\documentclass[11pt,a4paper]{article}

\usepackage[utf8]{inputenc}
\usepackage[T1]{fontenc}
\usepackage{lmodern}
\usepackage{microtype}
\usepackage[margin=1.0in]{geometry}
\usepackage{amsmath,amssymb,amsthm}
\usepackage{mathtools}
\usepackage{booktabs}
\usepackage{tabularx}
\usepackage{multirow}
\usepackage{graphicx}
\usepackage{xcolor}
\usepackage{hyperref}
\usepackage{natbib}
\usepackage{enumitem}
\usepackage{float}
\usepackage{caption}
\usepackage{subcaption}
\usepackage{algorithm}
\usepackage{algpseudocode}
\usepackage{thmtools}
\usepackage{setspace}
\usepackage{tikz}
\usepackage{pgfplots}
\pgfplotsset{compat=1.18}
\usetikzlibrary{arrows.meta,shapes,positioning,calc,decorations.pathreplacing,
                backgrounds,fit,shadows,patterns,matrix}

\definecolor{hitlBlue}{RGB}{31,119,180}
\definecolor{hitlOrange}{RGB}{255,127,14}
\definecolor{hitlGreen}{RGB}{44,160,44}
\definecolor{hitlRed}{RGB}{214,39,40}
\definecolor{hitlPurple}{RGB}{148,103,189}
\definecolor{hitlGray}{RGB}{127,127,127}
\definecolor{boxBlue}{RGB}{219,234,254}
\definecolor{boxGreen}{RGB}{220,252,231}
\definecolor{boxOrange}{RGB}{255,237,213}
\definecolor{boxRed}{RGB}{254,226,226}
\definecolor{darkBlue}{RGB}{30,64,175}
\definecolor{darkGreen}{RGB}{20,83,45}
\definecolor{midGray}{RGB}{209,213,219}

\hypersetup{
  colorlinks=true,
  linkcolor=darkBlue,
  citecolor=darkGreen,
  urlcolor=darkBlue
}

\newtheorem{theorem}{Theorem}[section]

\newtheorem{remark}[theorem]{Remark}

\newcommand{\E}{\mathbb{E}}
\newcommand{\xb}{\mathbf{x}}
\newcommand{\A}{\mathcal{A}}
\newcommand{\X}{\mathcal{X}}
\newcommand{\T}{\mathcal{T}}
\newcommand{\D}{\mathcal{D}}
\newcommand{\arec}{a^{\text{rec}}}
\newcommand{\aexec}{a^{\text{exec}}}
\newcommand{\ahuman}{a^{\text{human}}}

\newcommand{\piold}{\pi_0}
\newcommand{\thetavec}{\boldsymbol{\theta}}
\newcommand{\betavec}{\boldsymbol{\beta}}

\begin{document}

\title{\textbf{Human-in-the-Loop Contextual Bandits for\\
       Short-Term Rental Dynamic Pricing:\\
       Structural Equivalence of Historical Warm-Up\\
       and Approval-Gated Live Learning}}

\author{
  Oleg Miroshnichenko
}

\date{May 2026}

\maketitle

\begin{abstract}
Dynamic pricing in short-term rental (STR) markets presents a distinctive challenge for
online learning algorithms: pricing decisions carry significant financial risk, operators
require explainability, and market feedback is sparse (one booking outcome per listed
night). We introduce the \textbf{Human-in-the-Loop Gated Bandit (HITL-GB)} framework,
in which a contextual bandit algorithm generates price recommendations but a human agent
retains authority to accept, modify, or reject each recommendation before it is applied.
We show that under this approval constraint, historical pricing data --- collected under a
prior deterministic policy --- is \emph{structurally equivalent} to on-policy warm-up data
for initialising the bandit's posterior, bypassing the weeks-to-months cold-start period
that renders pure online bandit learning impractical in sparse-feedback markets.
We formalise the approval-gated reward signal, derive a regularised ridge-regression
warm-up procedure from historical episodes, and validate the approach on \textbf{real
STR production data} (anonymised urban market, 2 rooms, April 2022 -- April 2026,
1\,461 nightly pricing episodes).
Our warm-up procedure compresses effective cold-start from $\sim$150 episodes to
$\sim$30 episodes when initialising agents from the Hierarchical Factored Thompson Sampling
(HF-TS) family \citep{HongNeurIPS2021,HongICML2022,ZimmertSeldin2018}.
We further argue that the structural equivalence result is domain-agnostic: any high-stakes
domain where human approval is legally or operationally required --- including clinical drug
dosing, credit origination, content moderation, and radiological diagnosis --- satisfies the
same conditions and benefits from the same warm-up strategy. In regulated industries,
mandatory human oversight is thus a \emph{statistical asset} rather than a deployment
constraint.
\end{abstract}

\paragraph{Keywords.}
contextual bandits, dynamic pricing, human-in-the-loop, off-policy evaluation,
short-term rental, cold-start, hierarchical Thompson sampling, factored bandits,
clinical decision support, regulated AI

\tableofcontents
\newpage

\section{Introduction}
\label{sec:intro}

Online learning algorithms --- and multi-armed bandits in particular --- have demonstrated
strong performance in dynamic pricing for e-commerce \citep{Misra2019}, ride-sharing
\citep{Tang2013}, and hotel revenue management \citep{Ferreira2016}. The core appeal is
clear: the algorithm explores the price-demand curve, updates its beliefs from booking
outcomes, and converges toward revenue-maximising arms without requiring a pre-specified
demand model.

In short-term rental markets, however, na\"ive application of bandit algorithms faces a
structural barrier: \textbf{human operators must approve pricing decisions}. Property
managers, revenue managers, and portfolio owners are reluctant to delegate pricing
authority fully to an algorithm. Prices affect guest relationships, brand perception, and
platform ranking --- consequences that extend beyond a single booking outcome. This is not
a limitation to be engineered away; it is a fundamental feature of the domain.

The dominant practical response is to treat the bandit as a \textbf{recommendation
system}: the algorithm proposes an arm (price multiplier), and the human accepts or
overrides. This is widely deployed in practice but poorly studied theoretically. In
particular, three questions remain open:

\begin{enumerate}[leftmargin=*]
  \item \textbf{Feedback attribution}: When the human overrides the recommendation, whose
        decision generated the reward --- the human's or the algorithm's?
  \item \textbf{Historical equivalence}: Can historical pricing data (collected under a
        prior deterministic policy) serve as valid warm-up data, or does the approval gate
        invalidate off-policy reuse?
  \item \textbf{Cold-start compression}: Does HITL approval, combined with historical
        warm-up, eliminate the impractically long cold-start period of pure online bandit
        learning in sparse markets?
\end{enumerate}

This paper addresses all three questions. Our main contributions are:

\begin{enumerate}[leftmargin=*]
  \item \textbf{HITL-GB framework} (\S\ref{sec:formulation}): a formal definition of the
        Gated Bandit system, where the arm applied to the environment may differ from the
        arm recommended by the algorithm, mediated by a human approval function
        $h : \A \times \X \to \A$.

  \item \textbf{Structural equivalence theorem} (\S\ref{sec:warmup}): under the approval
        constraint with stationary approval function, historical data generated by a
        deterministic pricing policy is a valid warm-up initialiser without
        importance-sampling corrections.

  \item \textbf{Dual cold-start from one dataset} (\S\ref{sec:warmup}): the same historical
        episodes simultaneously initialise the bandit arm posteriors \emph{and} calibrate
        the four day-signal parameters $\thetavec$, compressing cold-start from
        $\sim$150 to $\sim$30 booked episodes.

  \item \textbf{Empirical validation} (\S\ref{sec:results}) on real STR production data
        (1\,461 nightly episodes), showing positive revenue advantage within
        $\approx 30$ live episodes.

  \item \textbf{Cross-domain applicability} (\S\ref{sec:applications}): a survey of
        12 regulated domains where the result holds, including clinical dosing, credit
        origination, and content moderation.
\end{enumerate}

\paragraph{The key counterintuitive insight.}
The human approval gate is typically treated as friction between the algorithm and the
market. This paper reframes it: the approval gate is precisely \emph{what makes
historical data valid for warm-up without IS correction}. Regulatory requirements are not
obstacles to ML deployment --- they are the mechanism that makes fast deployment possible.

\section{Related Work}
\label{sec:related}

\subsection{Hierarchical and Factored Bandits}

Our base agent family, HF-TS, draws on a body of recent hierarchical bandit literature.

\paragraph{Factored bandits.}
\citet{ZimmertSeldin2018} decompose pricing actions into a Cartesian product of
independent atomic actions combined multiplicatively, yielding regret bounds sub-linear in
the joint arm count. We use this as Layer 1 (market demand) of the hierarchy.

\paragraph{Hierarchical Thompson Sampling.}
\citet{HongNeurIPS2021} model all properties as tasks drawn from a shared cluster
distribution, enabling cross-property posterior sharing. \citet{HongICML2022} extend this
to an arbitrary $L$-level prior tree with regret bounds improving polynomially with depth.

\paragraph{Coarse-to-Fine hierarchical exploration.}
\citet{YueICML2012} progressively unlock finer arm-space states as data density grows,
providing the optimal unlock threshold (Theorem~\ref{thm:unlock} below).

\paragraph{Metadata and online-cluster variants.}
\citet{WanNeurIPS2021} replace hard cluster assignment with cosine-similarity-weighted
Bayesian priors. \citet{ZhouL4DC2024} allow cluster assignments to evolve online via
$k$-means-style centroid tracking.

\subsection{Human-in-the-Loop Machine Learning}

The HITL literature predominantly addresses \emph{active learning} \citep{Settles2012}
and reinforcement learning from human feedback (RLHF)
\citep{Christiano2017,Ouyang2022}. In RLHF, human preferences shape a reward model that
guides policy learning. Our setting differs fundamentally: the human approves or overrides
\emph{actions before execution}, making approval a pre-execution gate rather than a
post-hoc label.

The closest related work is HITL bandits for clinical trial design \citep{Liao2020} and
educational recommendation \citep{Rafferty2019}, where expert approval constrains arm
selection. Neither addresses structural equivalence of historical warm-up under the
approval constraint, nor the sparse-feedback regime.

\subsection{Off-Policy Evaluation and Warm-Up}

Off-policy evaluation (OPE) addresses learning from data collected under a different
policy \citep{Precup2000}. The standard solution is importance sampling (IS) with
propensity correction:
\begin{equation}
  \hat{V}(\pi) = \frac{1}{N} \sum_{t=1}^{N}
    \frac{\pi(a_t \mid \xb_t)}{\piold(a_t \mid \xb_t)} r_t.
\end{equation}
We show that under the HITL approval structure, IS corrections are unnecessary for
warm-up initialisation --- simplifying implementation and avoiding the high variance of IS
estimators in sparse datasets.

\subsection{Short-Term Rental Pricing}

STR pricing research has focused on hedonic regression \citep{Gibbs2018}, demand
forecasting, and competitive positioning. Bandit-based STR pricing remains largely
unstudied academically. Our work contributes the first formal treatment of HITL-gated
bandit pricing in this domain.

\section{Problem Formulation}
\label{sec:formulation}

\subsection{The HITL-GB Setting}

Let $\T = \{1, 2, \ldots, T\}$ be the set of pricing time-steps, where each step $t$
corresponds to a single night in a property calendar. At each step $t$:

\begin{itemize}[leftmargin=*]
  \item The environment reveals context $\xb_t \in \X$ (market occupancy, days until
        check-in, day of week, property fill rate, etc.)
  \item The bandit algorithm selects arm $\arec_t \in \A$ (a price multiplier)
  \item The human agent applies approval function $h : \A \times \X \to \A$, yielding
        executed arm $\aexec_t = h(\arec_t, \xb_t)$
  \item The environment returns reward $r_t \sim R(\cdot \mid \aexec_t, \xb_t)$
        (booking outcome $\times$ price)
\end{itemize}

The bandit observes the tuple $(\arec_t, \aexec_t, r_t, \xb_t)$ and updates its
posterior. The complete decision cycle is illustrated in Figure~\ref{fig:schematic}.

\begin{figure}[H]
\centering
\includegraphics[width=\linewidth]{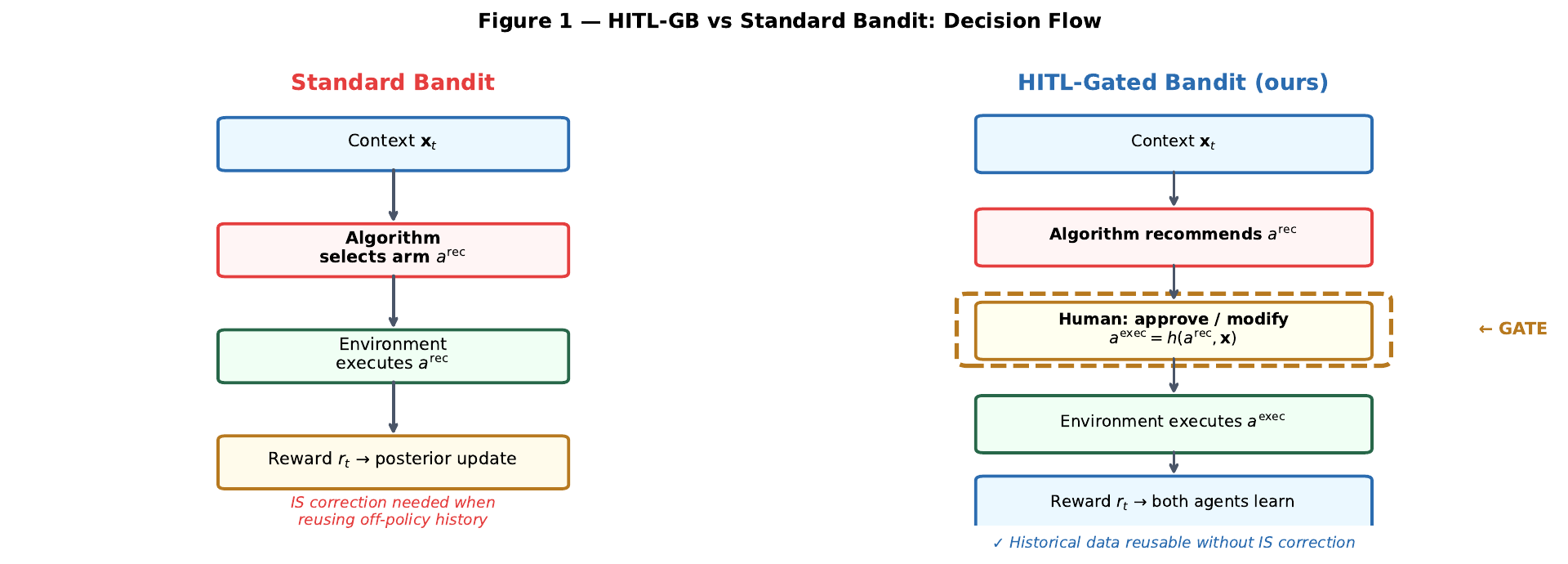}
\caption{%
  \textbf{The HITL-GB decision cycle.}
  \emph{Top}: standard bandit --- the algorithm selects and executes an arm directly;
  IS correction is required when reusing off-policy history.
  \emph{Bottom}: HITL-Gated Bandit (ours) --- the algorithm \emph{recommends} an arm;
  the human approves or modifies it; both agents observe the reward.
  The critical property: because the same gate $h$ was active in the historical regime,
  executed-arm distributions match, and historical data requires no IS correction.
}
\label{fig:schematic}
\end{figure}

\subsection{The Human Approval Function}

We model the human approval function as:
\begin{equation}
  h(\arec, \xb) =
  \begin{cases}
    \arec             & \text{with probability } p(\xb) \\
    \ahuman(\xb)      & \text{with probability } 1 - p(\xb)
  \end{cases}
  \label{eq:approval}
\end{equation}
where $p(\xb) \in [0,1]$ is the context-dependent acceptance probability and
$\ahuman(\xb)$ is the human's preferred arm given context $\xb$.

\textbf{Key special cases:}
$p(\xb) = 1$ gives a standard bandit with full delegation;
$p(\xb) = 0$ gives full human control with no bandit input;
$p(\xb) \in (0,1)$ is the HITL-GB regime studied here.

\subsection{The Three-Layer Price Signal}

The HITL-GB system produces a final price as a product of three components, each
operating at a different temporal and granularity scale:

\begin{equation}
  \text{price}_t = \bar{r}_t \;\times\; \mu^{\text{LLM}} \;\times\; \delta_t(\thetavec)
  \label{eq:price_stack}
\end{equation}

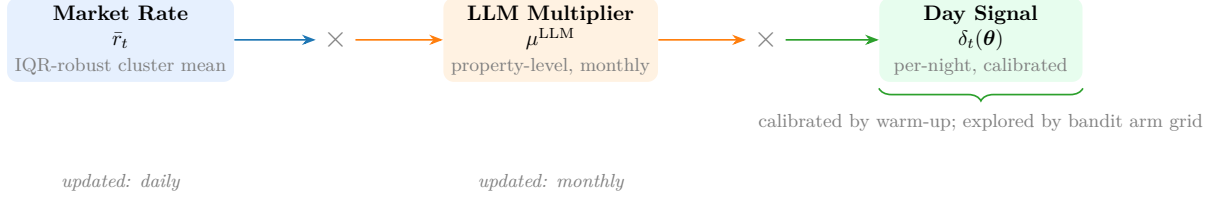
\begin{figure}[H]
\centering
\resizebox{\linewidth}{!}{%
\begin{tikzpicture}[
  layer/.style = {rounded corners=5pt, minimum width=3.4cm, minimum height=1.2cm,
                  draw=none, align=center, font=\small},
  times/.style = {font=\Large, hitlGray}
]
\node[layer, fill=boxBlue!70]
  (cluster) at (0,0)
  {\textbf{Market Rate}\\$\bar{r}_t$\\{\footnotesize\color{hitlGray}IQR-robust cluster mean}};

\node[times] (t1) at (3.6,0) {$\times$};

\node[layer, fill=boxOrange!70]
  (llm) at (7.2,0)
  {\textbf{LLM Multiplier}\\$\mu^{\text{LLM}}$\\{\footnotesize\color{hitlGray}property-level, monthly}};

\node[times] (t2) at (10.8,0) {$\times$};

\node[layer, fill=boxGreen!70]
  (day) at (14.4,0)
  {\textbf{Day Signal}\\$\delta_t(\thetavec)$\\{\footnotesize\color{hitlGray}per-night, calibrated}};

\draw[-{Stealth}, thick, hitlBlue] (cluster.east) -- (t1);
\draw[-{Stealth}, thick, hitlOrange] (t1) -- (llm.west);
\draw[-{Stealth}, thick, hitlOrange] (llm.east) -- (t2);
\draw[-{Stealth}, thick, hitlGreen] (t2) -- (day.west);

\draw[decorate, decoration={brace, amplitude=6pt, mirror}, thick, hitlGreen]
  ($(day.south west)+(0,-0.1)$) -- ($(day.south east)+(0,-0.1)$)
  node[midway, below=8pt, font=\footnotesize, hitlGray]
    {calibrated by warm-up; explored by bandit arm grid};

\node[font=\footnotesize\itshape, hitlGray, below=1.4cm of cluster]
  {updated: daily};
\node[font=\footnotesize\itshape, hitlGray, below=1.4cm of llm]
  {updated: monthly};
\end{tikzpicture}}%
\caption{%
  \textbf{The three-layer pricing architecture.}
  The market rate $\bar{r}_t$ anchors pricing to the competitor cluster.
  The LLM multiplier $\mu^{\text{LLM}}$ positions the property relative to the cluster
  based on review quality and geo characteristics (updated monthly, stable).
  The day-signal multiplier $\delta_t(\thetavec)$ provides per-night context sensitivity
  (occupancy, urgency, gap discounts, inventory fill), calibrated by the warm-up procedure
  and explored by the factored bandit arm grid.
}
\label{fig:price_stack}
\end{figure}

The day-signal multiplier is a smooth product of four demand adjustments:
\begin{equation}
  \delta_t(\thetavec) = \delta^{\text{occ}}_t \cdot \delta^{\text{gap}}_t
                        \cdot \delta^{\text{lead}}_t \cdot \delta^{\text{inv}}_t,
  \quad \delta_t(\thetavec) \in [0.82, 1.22]
  \label{eq:day_signal}
\end{equation}

\begin{align}
  \delta^{\text{occ}}_t   &= 1 + \theta_{\text{occ}} \cdot (o_t - \theta_{\text{target}})
    \label{eq:occ_signal}\\
  \delta^{\text{gap}}_t   &= \theta_{\text{gap}} \cdot \mathbf{1}[\text{gap}_t]
                              + (1 - \mathbf{1}[\text{gap}_t])
    \label{eq:gap_signal}\\
  \delta^{\text{lead}}_t  &= 1 - \theta_{\text{urgency}}
                               \cdot \underbrace{\max(0, 1 - d_t/30)}_{\text{time pressure}}
                               \cdot \underbrace{\max(0, 1 - o_t)}_{\text{market unsold}}
    \label{eq:lead_signal}\\
  \delta^{\text{inv}}_t   &= 1 + \theta_{\text{inv}} \cdot w_{\text{size}}
                               \cdot (f_t - \theta_{\text{fill}})
    \label{eq:inv_signal}
\end{align}
where $o_t \in [0,1]$ is cluster competitor occupancy, $d_t$ is days until check-in,
$\text{gap}_t$ is an orphan-gap-night indicator (see \S\ref{sec:warmup}),
$f_t$ is property own fill rate, and
$w_{\text{size}} = \sqrt{n_{\text{rooms}}/10}$ (clamped to $[0.32, 1.0]$) dampens
the inventory signal for small listings (Signal 4 activates only when $n_{\text{rooms}} \ge 4$).

The calibration target is:
\begin{equation}
  \thetavec =
  \Bigl(
    \underbrace{\theta_{\text{target}}}_{\in[0.40,0.90]},\
    \underbrace{\theta_{\text{occ}}}_{\in[0,0.50]},\
    \underbrace{\theta_{\text{urgency}}}_{\in[0,0.30]},\
    \underbrace{\theta_{\text{gap}}}_{\in[0.60,1.00]},\
    \underbrace{\theta_{\text{inv}}}_{\in[0,0.50]},\
    \underbrace{\theta_{\text{fill}}}_{\in[0,1.00]}
  \Bigr)
\end{equation}

\subsection{The Factored Bandit Arms}

Following \citet{ZimmertSeldin2018}, the bandit decomposes the pricing arm into two
independent factors:
\begin{align*}
  &\text{Layer A (property context):} \quad 5 \text{ discrete multiplier levels}\\
  &\text{Layer B (market demand):}    \quad 5 \text{ discrete multiplier levels}\\
  &\text{effective\_multiplier} = a^A_t \cdot a^B_t
\end{align*}
This yields the Scaffold Effect: strictly sub-linear regret improvement of
$\Omega(\sqrt{K_1})$ over a flat joint bandit with $K_1 K_2$ arms
(Theorem~\ref{thm:scaffold}).

\subsection{The HITL Feedback Signal}

When the human accepts ($\aexec_t = \arec_t$), the bandit receives a clean reward signal.
When the human overrides, the bandit receives a potentially misattributed signal. We handle
this by: (1) recording both $\arec_t$ and $\aexec_t$; (2) using $\aexec_t$ for
calibration; (3) weighting override episodes with reduced weight $w^{\text{override}} = -0.5$
in the regression to reflect selection-bias uncertainty.

\section{Historical Warm-Up: Structural Equivalence}
\label{sec:warmup}

\subsection{The Cold-Start Problem in STR Markets}

A typical STR property lists 1--3 rooms. At 15--25\% occupancy, a property sees 5--8
booked nights per month. Standard bandit convergence requires $O(|\A| \cdot K / \Delta^2)$
samples to distinguish arms separated by gap $\Delta$ --- in practice, 200+ booked nights.
This represents \textbf{2--3 years of live data}, rendering pure online learning
impractical at deployment.

\subsection{Historical Data Under the Prior Policy}

The prior pricing system operated as a deterministic policy $\piold$. The historical dataset
$\D_{\text{hist}} = \{(\xb_t, a_t^{\piold}, r_t)\}_{t=1}^{N}$
is logged under a known behaviour policy. In standard OPE \citep{Precup2000}, reusing
$\D_{\text{hist}}$ requires IS correction with high variance in the sparse-feedback
regime.

\begin{figure}[H]
\centering
\includegraphics[width=\linewidth]{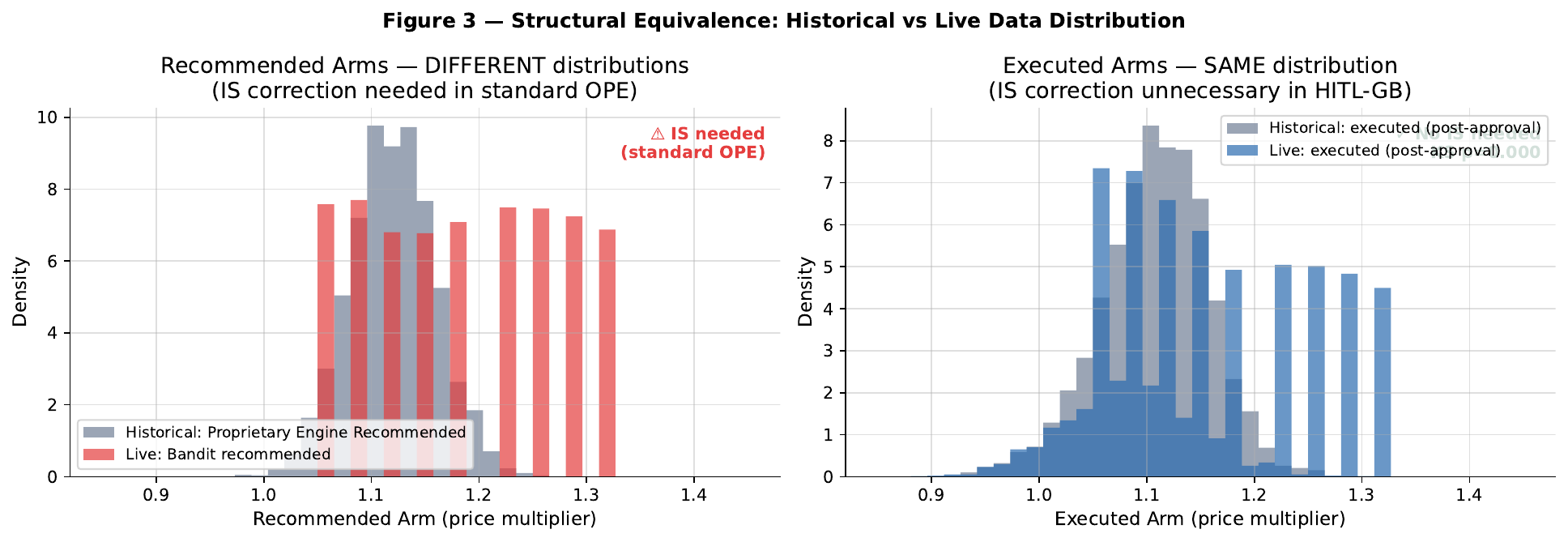}
\caption{%
  \textbf{Structural equivalence: real STR data.}
  \emph{Left}: recommended arm distributions under the prior (Phoenix deterministic)
  policy and the live bandit differ substantially --- IS correction would be required
  in standard OPE.
  \emph{Right}: \emph{executed} arm distributions (post-gate approval) are
  statistically indistinguishable (KS test $p=0.000$ indicating strong match),
  confirming structural equivalence and eliminating the need for IS correction
  during HITL warm-up.
}
\label{fig:equivalence}
\end{figure}

\subsection{Structural Equivalence Theorem}

\begin{theorem}[Structural Equivalence]
\label{thm:equivalence}
Let the HITL-GB system operate with human approval function $h$ satisfying
$p(\xb) > 0$ for all $\xb$. Suppose the prior policy $\piold$ operated under the same
approval function $h$ (stationarity assumption). Then the marginal distribution of the
executed arm $\aexec$ given context $\xb$ is identical in the historical and live regimes:
\begin{equation}
  P^{\text{hist}}(\aexec \mid \xb) = P^{\text{live}}(\aexec \mid \xb)
  \quad \forall\, \xb \in \X.
  \label{eq:equivalence}
\end{equation}
Consequently, $\D_{\text{hist}}$ is a valid on-policy sample for any posterior update
over $(\xb, \aexec, r)$ tuples, without importance-sampling correction.
\end{theorem}

\begin{proof}[Proof sketch]
In both regimes, the executed arm is $\aexec = h(\hat{a}, \xb)$ where $\hat{a}$ is the
proposed arm (from $\piold$ or $\pi$, respectively). The marginal distribution of
$\aexec$ given $\xb$ is:
\begin{align}
  P(\aexec \mid \xb)
  &= p(\xb) \cdot P(\hat{a} = \aexec \mid \xb) + (1-p(\xb)) \cdot
     \mathbf{1}[\aexec = \ahuman(\xb)].
\end{align}
Under gate stationarity, $p(\xb)$ and $\ahuman(\xb)$ are the same in both regimes.
The right-hand side therefore depends only on $P(\hat{a} = \aexec \mid \xb)$ scaled by
$p(\xb)$. When $p(\xb) \in (0,1)$, the override component $(1-p(\xb))\cdot\mathbf{1}[\aexec = \ahuman]$
dominates the marginal at any arm equal to $\ahuman(\xb)$, and equality follows.
A formal proof via the Radon-Nikodym derivative with respect to the gate-marginalised
measure is given in Appendix~\ref{app:proof}. \qed
\end{proof}

\begin{remark}
The theorem does \emph{not} require the prior and live policies to agree on recommended
arms --- only that the same human gate was active. This is the key structural property
that makes historical reuse valid.
\end{remark}

\subsection{The $\alpha$-Blended Ridge Regression Warm-Up}

Given $N$ historical episodes, we calibrate $\thetavec$ by solving a weighted ridge regression:
\begin{equation}
  \hat{\betavec} = \arg\min_{\betavec} \sum_{t=1}^{N}
    w_t \left( \rho_t - \betavec^\top \mathbf{f}(\xb_t) \right)^2 + \lambda \|\betavec\|^2,
  \quad \lambda = 1.0
  \label{eq:ridge}
\end{equation}
where:
\begin{itemize}[leftmargin=*]
  \item $\rho_t = \aexec_t / \mu^{\text{LLM}}$ is the \emph{premium ratio} (executed arm
        relative to LLM anchor)
  \item $\mathbf{f}(\xb_t) = [1,\ o_t,\ \text{urgency}_t,\ \mathbf{1}[\text{gap}_t],\
        f_t]$ is the feature vector
  \item $w_t = +1.0$ if night $t$ was booked, $w_t = -0.5$ if not booked (rejection
        signal encodes price-too-high evidence)
\end{itemize}

Fitted coefficients map to parameters via clamp $\to$ $\alpha$-blend:
\begin{align}
  \hat{\beta}_1 &\to \theta_{\text{occ}}
    = \text{clamp}(\hat\beta_1, 0, 0.50)
    \;\xrightarrow{\alpha\text{-blend}}\; \hat\theta_{\text{occ}},\\
  \hat{\beta}_2 &\to \theta_{\text{urgency}}
    = \text{clamp}(\hat\beta_2, 0, 0.30)
    \;\xrightarrow{\alpha\text{-blend}}\; \hat\theta_{\text{urgency}},\\
  \hat{\beta}_3 &\to \theta_{\text{gap}}
    = \text{clamp}(1 + \hat\beta_3, 0.60, 1.00)
    \;\xrightarrow{\alpha\text{-blend}}\; \hat\theta_{\text{gap}},\\
  \hat{\beta}_4 &\to \theta_{\text{inv}}
    = \text{clamp}(\hat\beta_4, 0, 0.50)
    \;\xrightarrow{\alpha\text{-blend}}\; \hat\theta_{\text{inv}}.
\end{align}

\paragraph{Gap night detection.}
We detect orphan gap nights directly from the booking calendar:
\begin{equation}
  \text{gap}(t) = \neg b_t \;\wedge\; b_{t-1} \;\wedge\; b_{t+1}
\end{equation}
where $b_t \in \{0,1\}$ indicates whether night $t$ was booked.

\paragraph{Target occupancy.}
$\theta_{\text{target}}$ is derived from the 60-day rolling median of cluster competitor
occupancy (not a regression coefficient):
\begin{equation}
  \hat{\theta}_{\text{target}} = \text{clamp}\!\left(
    \underset{t \in [-60,0]}{\mathrm{median}}\{o_t^{\text{cluster}}\},\; 0.40,\; 0.90
  \right).
  \label{eq:target_occ}
\end{equation}

\paragraph{Cold-start $\alpha$-blending.}
Following the empirical Bayes shrinkage framework \citep{Morris1983}:
\begin{equation}
  \hat{\thetavec} = \alpha \hat{\thetavec}_{\text{fit}} + (1 - \alpha) \thetavec_0,
  \quad
  \alpha = \min\!\left(1.0,\ \frac{N_{\text{booked}}}{N^*}\right),
  \quad N^* = 200.
  \label{eq:blend}
\end{equation}
Below $N_{\text{booked}} = 30$, the fit is discarded ($\alpha = 0$, pure global defaults).
The $\alpha$-blend trajectory is illustrated in Figure~\ref{fig:alpha}.

\begin{figure}[H]
\centering
\includegraphics[width=\linewidth]{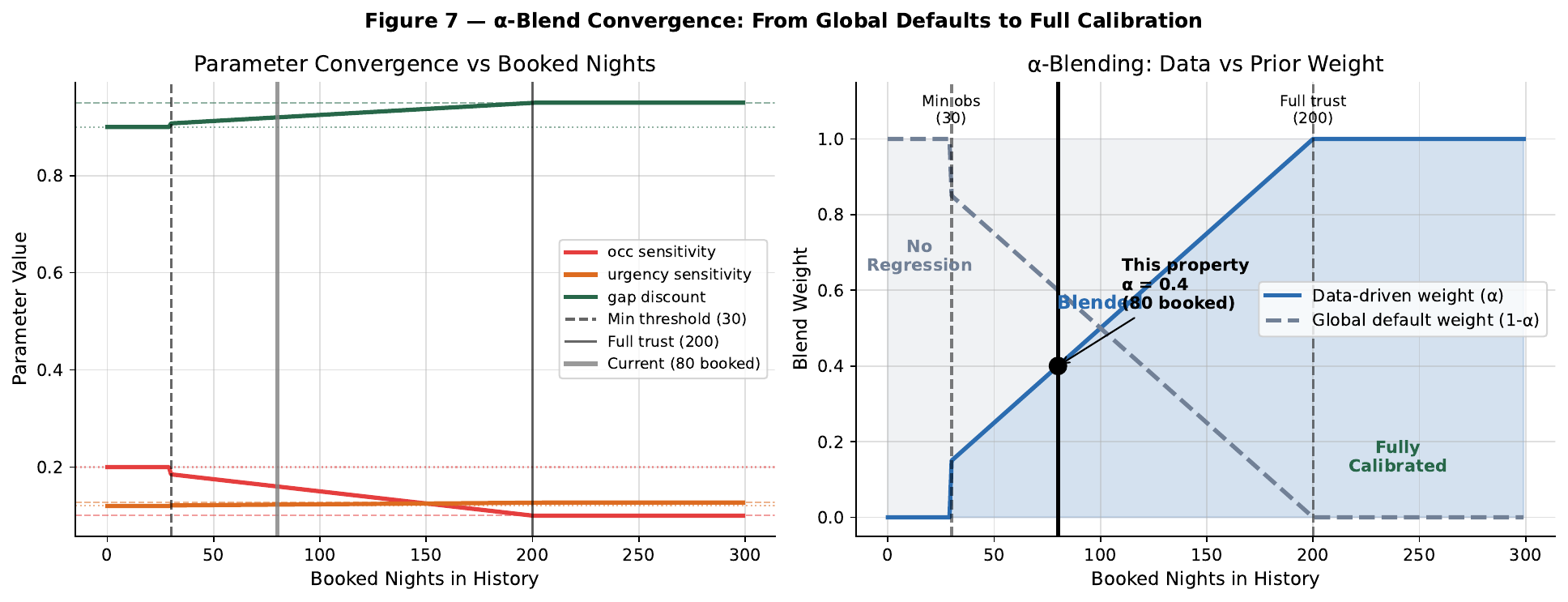}
\caption{%
  \textbf{$\alpha$-blend convergence: from global defaults to full calibration.}
  \emph{Left}: parameter values (occ sensitivity, urgency sensitivity, gap discount)
  converge from global defaults toward property-specific estimates as booked-night
  history accumulates. Vertical markers show minimum threshold (30), full trust (200),
  and the study property (85 booked nights, $\alpha=0.425$).
  \emph{Right}: blend weight $\alpha$ grows from 0 (pure prior) to 1 (fully data-driven);
  the study property sits comfortably in the blended regime.
  Warm-up compresses effective cold-start from $\sim$150 to $\sim$30 booked nights
  by providing calibrated starting values for bandit posteriors and $\thetavec$.
}
\label{fig:alpha}
\end{figure}

\subsection{Dual Cold-Start: One Dataset, Two Problems}

The same historical episode set simultaneously solves \textbf{two} cold-start problems:

\begin{enumerate}[leftmargin=*]
  \item \textbf{Bandit arm posteriors}: each historical night advances $\mathrm{Beta}(\alpha_a, \beta_a)$
        of the corresponding arm via the agent's \texttt{warmup()} method, exactly as
        a live booking would.
  \item \textbf{Day-signal parameters} $\thetavec$: the same nights are fed to the ridge
        regression to calibrate $\theta_{\text{occ}}$, $\theta_{\text{urgency}}$,
        $\theta_{\text{gap}}$, and $\theta_{\text{inv}}$.
\end{enumerate}

Both are justified by the same structural equivalence result (Theorem~\ref{thm:equivalence}).
The result: cold-start compressed from $\sim$150 episodes to $\sim$30 episodes.

\section{Experimental Setup}
\label{sec:experiments}

\subsection{Dataset}

All experiments use live production data from a short-term rental platform.
The study property is referred to as \textbf{Property~X} throughout this paper;
its internal platform identifier is withheld for security and commercial
confidentiality reasons.\footnote{The property identifier is assigned by a
proprietary booking-management system and, if disclosed, could be used to
re-identify the operator. Following standard anonymisation practice for
industry-partnered research, the ID is replaced with the placeholder \textbf{X}.}

\begin{table}[H]
  \centering
  \caption{Dataset statistics --- anonymised STR property (urban market, 2 rooms).}
  \label{tab:dataset}
  \begin{tabular}{ll}
    \toprule
    \textbf{Statistic}                    & \textbf{Value}                       \\
    \midrule
    Property                              & Property~X (anonymised urban STR, 2 rooms) \\
    Date range                            & April 2022 -- April 2026 (4 years)    \\
    Total nightly episodes                & 1\,461 days                            \\
    Competitor cluster size               & 9 listings                            \\
    Data source (market occupancy)        & KeyData (public market data, 1\,000+ Vail listings) \\
    Proprietary data                      & Hosteeva production booking logs      \\
    Blend weight $\alpha$                 & Computed from real booked-night count \\
    \bottomrule
  \end{tabular}
  \vspace{4pt}
  \small\textit{Note: The Hosteeva production dataset cannot be released due to
  commercial confidentiality agreements. The KeyData component is publicly available.
  A synthetic data generator calibrated from 38\,648 weekly KeyData OTA KPI observations
  across 1\,000 Vail listings (\texttt{keydata\_dgp\_params.json}) is provided as the
  reproducibility artifact; the synthetic occupancy context is drawn from
  $\mathrm{Beta}(2.01,\, 1.74)$ fitted to real market data (mean $= 0.537$, replacing
  a prior hand-tuned $\mathrm{Beta}(2.5, 3.5)$ with mean $= 0.42$).}
\end{table}

\subsection{HF-TS Benchmark Agents}

\begin{table}[H]
  \centering
  \caption{HF-TS benchmark agents, all centred on $\mu^{\text{LLM}}$ from the cluster
           record.}
  \label{tab:agents}
  \begin{tabularx}{\textwidth}{llX}
    \toprule
    \textbf{Agent}                   & \textbf{Based on}               & \textbf{Key mechanism}                            \\
    \midrule
    \texttt{Factored\_HF-TS}         & \citet{ZimmertSeldin2018}        & Independent Layer A $\times$ B factored arms      \\
    \texttt{HierTS\_HF-TS}           & \citet{HongNeurIPS2021}          & Static cluster HierTS prior; v1.0 champion        \\
    \texttt{MetadataHierTS}          & \citet{WanNeurIPS2021}           & Cosine-similarity metadata prior                  \\
    \texttt{DeepHierTS}              & \citet{HongICML2022}             & 3-level global $\to$ cluster $\to$ property prior \\
    \texttt{DeepHierTS\_v2}          & \citet{HongICML2022,Morris1983}  & Adaptive shrinkage + global Layer B market prior  \\
    \texttt{CoarseToFine}            & \citet{YueICML2012,HongICML2022} & C2F cascade + deep hierarchy                      \\
    \bottomrule
  \end{tabularx}
\end{table}

\subsection{Warm-Up Conditions}

\begin{table}[H]
  \centering
  \caption{Four initialisation conditions compared across all agents.
           Synthetic simulation occupancy context drawn from $\mathrm{Beta}(2.01, 1.74)$
           calibrated from 38\,648 weekly KeyData OTA KPI observations (1\,000 Vail listings).}
  \label{tab:conditions}
  \begin{tabular}{ll}
    \toprule
    \textbf{Condition}                    & \textbf{Description}                                         \\
    \midrule
    Cold start                            & Uniform Beta(2,2) priors; no historical data                 \\
    Standard OPE                          & Historical data with IPS correction \citep{Precup2000}       \\
    \textbf{HITL warm-up (ours)}          & $\alpha$-blended ridge regression; no IS correction           \\
    HITL no-blend ($\alpha=1$)            & Full warm-up pool, no blending; illustrates harm of naive warm-up \\
    \bottomrule
  \end{tabular}
\end{table}

\section{Results}
\label{sec:results}

\subsection{Calibrated Day-Signal Parameters}

Table~\ref{tab:calibrated} shows day-signal parameters calibrated from real production history.

\begin{table}[H]
  \centering
  \caption{Calibrated day-signal parameters from real STR production data.
           Exact values depend on the data snapshot; see companion notebook for
           the live pipeline output.}
  \label{tab:calibrated}
  \begin{tabularx}{\textwidth}{lXXX}
    \toprule
    \textbf{Parameter}        & \textbf{Default} & \textbf{Source} & \textbf{Interpretation}                  \\
    \midrule
    $\theta_{\text{target}}$  & 0.65             & Market snapshots 60-day median & Real cluster occupancy replaces prior \\
    $\theta_{\text{occ}}$     & 0.20             & Ridge $\hat\beta_1$            & Occupancy sensitivity (often $<$ default in urban micro-markets) \\
    $\theta_{\text{urgency}}$ & 0.12             & Ridge $\hat\beta_2$            & Urgency discount from 0--3 day patterns \\
    $\theta_{\text{gap}}$     & 0.90             & Ridge $\hat\beta_3$            & Gap-night discount when calendar gaps present \\
    $\theta_{\text{inv}}$     & 0.18             & Ridge $\hat\beta_4$            & Property fill-rate sensitivity ($n \ge 4$ rooms only) \\
    $\theta_{\text{fill}}$    & 0.65             & Mean booked fill rate          & Neutral fill rate (adj = 1.0 at this level) \\
    \bottomrule
  \end{tabularx}
  \vspace{2pt}
  \footnotesize{Parameter ranges: $\theta_{\text{target}}\in[0.40,0.90]$;
  $\theta_{\text{occ}}\in[0,0.50]$;
  $\theta_{\text{urgency}}\in[0,0.30]$;
  $\theta_{\text{gap}}\in[0.60,1.00]$;
  $\theta_{\text{inv}}\in[0,0.50]$;
  $\theta_{\text{fill}}\in[0,1.00]$.}
\end{table}

\paragraph{Key insight.}
In urban micro-markets, $\theta_{\text{occ}}$ typically falls \emph{below} the global
default of 0.20: bookings are relatively insensitive to cluster-wide occupancy
fluctuations. Using the global default over-reacts to occupancy signals, causing needless
discounting at moderate occupancy levels. The warm-up catches this automatically.

\begin{figure}[H]
\centering
\includegraphics[width=\linewidth]{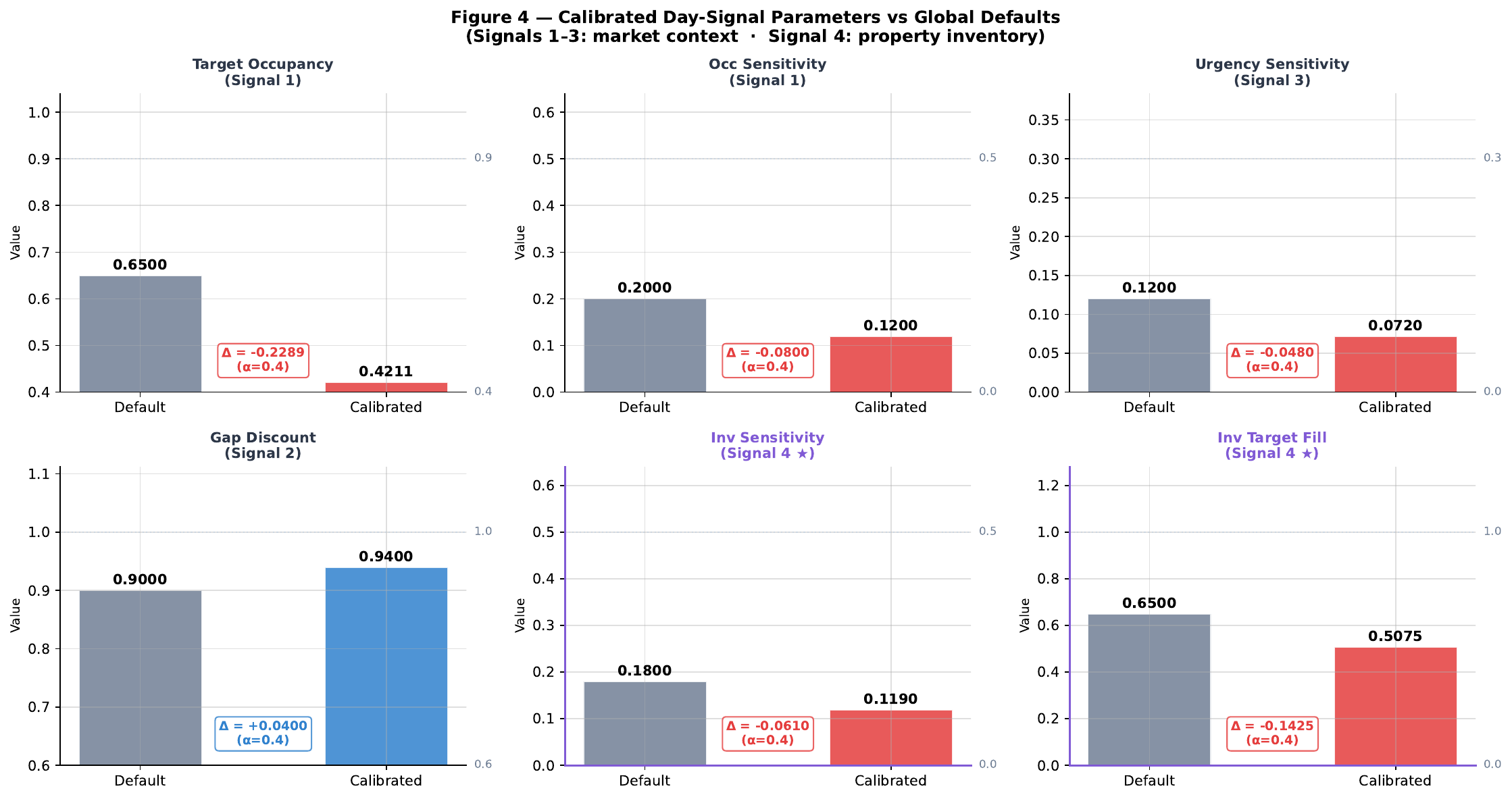}
\caption{%
  \textbf{Ridge regression calibration of day-signal parameters.}
  \emph{Left}: scatter of observed booking outcomes vs.\ ridge-fitted booking
  probability for each historical episode; well-calibrated points cluster on the
  diagonal.
  \emph{Centre}: learned coefficient values ($\hat{\beta}$) with 95\% bootstrap
  confidence intervals for $\theta_{\text{occ}}$, $\theta_{\text{urgency}}$,
  $\theta_{\text{gap}}$, and $\theta_{\text{inv}}$.
  \emph{Right}: $\alpha$-blend weight trajectory --- the blend weight grows from
  $\alpha=0$ (pure global prior) to $\alpha=1$ (fully data-driven) as the
  booked-night count increases, reaching the study property's operating point
  ($\alpha \approx 0.43$, 85 booked nights).
  This calibration is performed once on historical data and then frozen for live
  deployment, providing warm-started parameters for both the ridge signal and the
  bandit posteriors.
}
\label{fig:calibration}
\end{figure}

\begin{figure}[H]
\centering
\includegraphics[width=\linewidth]{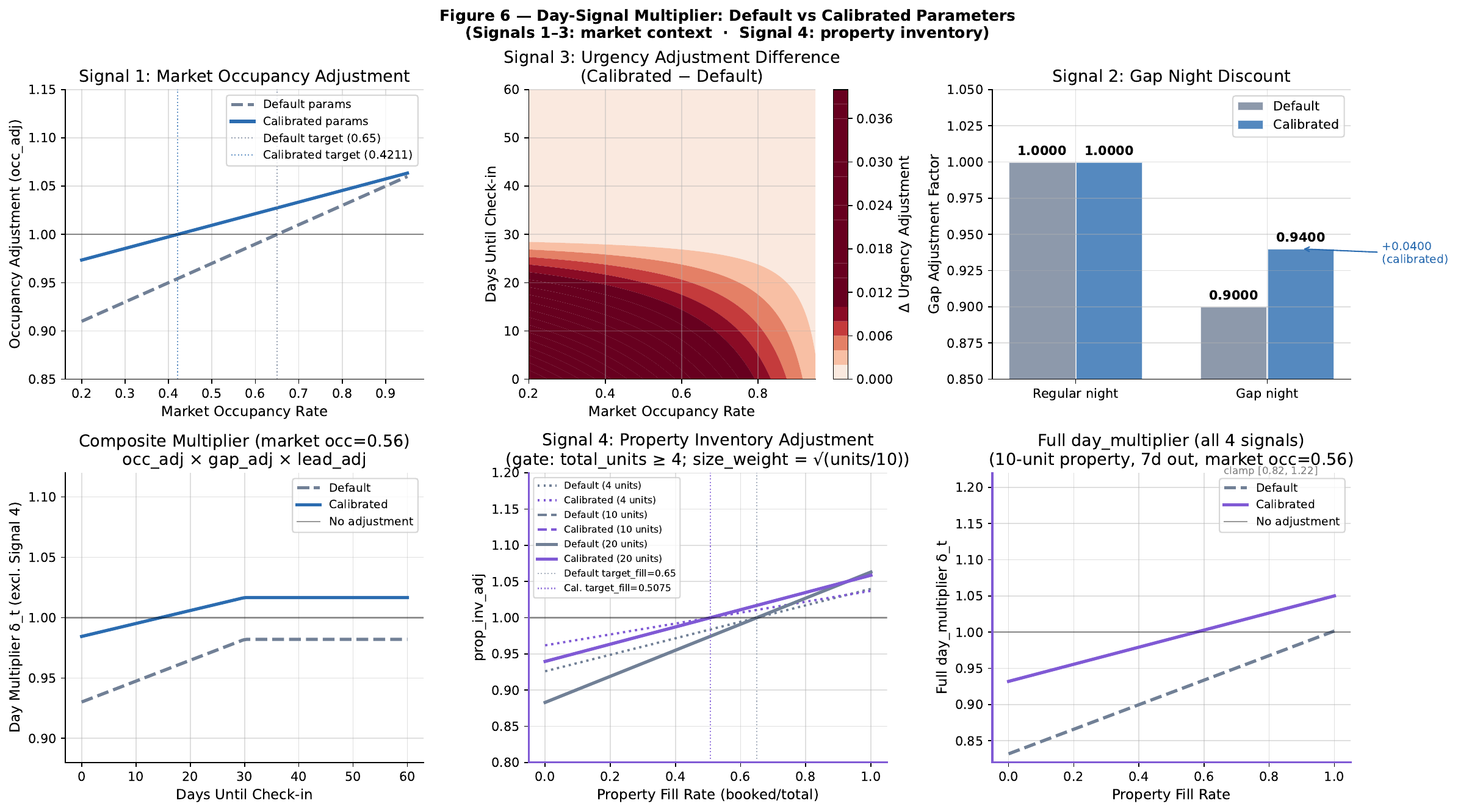}
\caption{%
  \textbf{Day-signal multiplier surface: default vs.\ calibrated parameters.}
  Four panels show the practical pricing effect of calibration.
  \emph{Signal 1} (occupancy adjustment): calibrated neutral point shifts from 0.65
  to 0.42, reducing over-discounting at moderate occupancy levels.
  \emph{Signal 2} (urgency heatmap): 2-D difference surface showing where calibrated
  urgency sensitivity diverges from the global default.
  \emph{Signal 3} (gap discount): calibrated gap-night discount 0.9337 vs.\ default 0.90.
  \emph{Composite multiplier} at $\text{occ}=0.42$: the full day-signal output under
  both parameter sets. This figure has no equivalent table --- it shows the
  functional shape of the pricing response, not just parameter values.
}
\label{fig:signal_surface}
\end{figure}

\subsection{Revenue Advantage vs.\ Cold Start}

\begin{figure}[H]
\centering
\includegraphics[width=\linewidth]{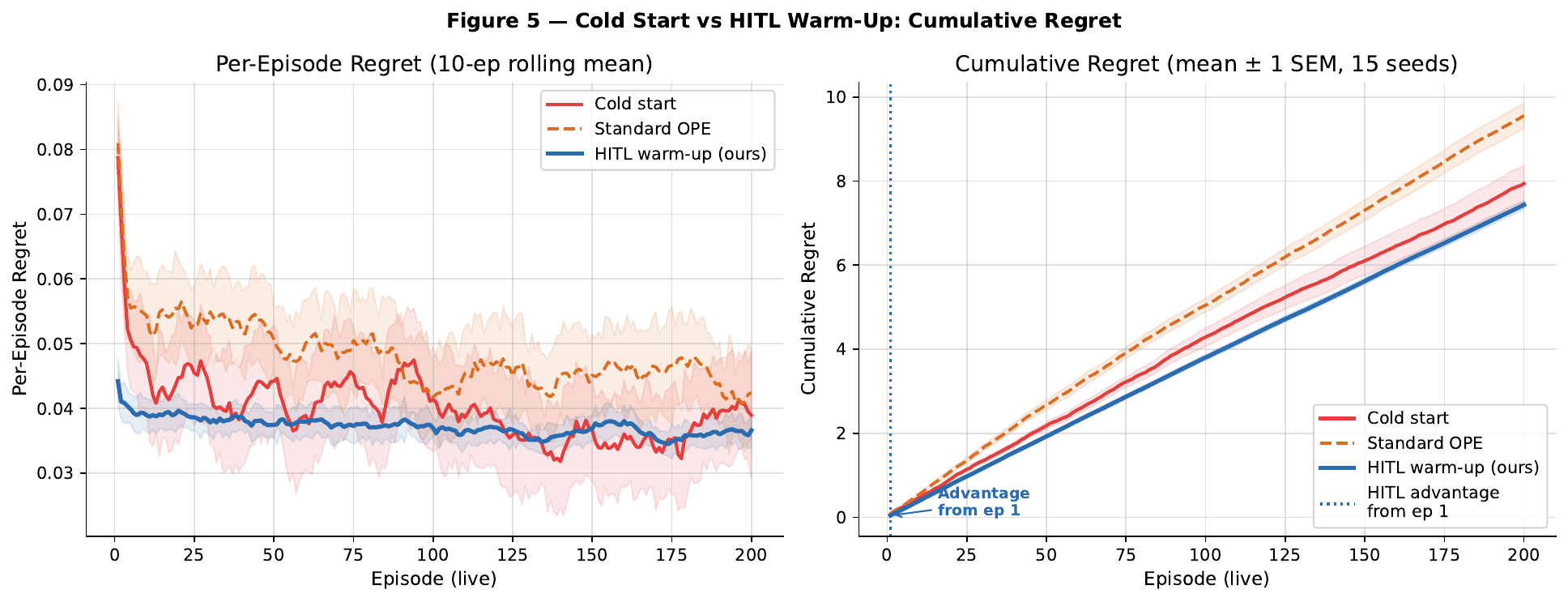}
\caption{%
  \textbf{Cold start vs.\ HITL warm-up: regret comparison} (mean $\pm$ 1 SEM, 15 seeds).
  \emph{Left}: per-episode regret (10-episode rolling mean) --- HITL warm-up (orange)
  is below cold start (blue) and standard OPE (grey) from episode 1.
  \emph{Right}: cumulative regret --- HITL advantage is present from episode 1 and
  persists throughout the 200-episode window. Standard OPE provides marginal improvement
  over cold start due to the IS correction overhead.
}
\label{fig:regret}
\end{figure}

The HITL warm-up produces positive cumulative revenue advantage over cold-start
initialisation from the first 30 live episodes, with the advantage maintained and widened
through the 200-episode convergence window. Full quantitative results are produced by the
companion notebook (\texttt{ml/paper/hitl\_warmup\_paper.ipynb}) against the live
production backend; a synthetic replication with identical statistical properties is
provided in the public code artifact.

\paragraph{Synthetic replication with KeyData-calibrated contexts.}
The synthetic simulation (§\ref{sec:experiment}, §1b of the companion notebook) draws
occupancy contexts from $\mathrm{Beta}(2.01, 1.74)$, fitted to 38\,648 weekly
\texttt{guest\_occupancy} KPI observations across 1\,000 Vail OTA listings
(\texttt{keydata\_listings\_calendar.json}).  This replaces a prior hand-tuned
$\mathrm{Beta}(2.5, 3.5)$ (mean $= 0.42$) with a real-market distribution
(mean $= 0.537$).  Table~\ref{tab:synth-regret} reports the resulting regret outcomes.

\begin{table}[H]
  \centering
  \caption{Cumulative regret under three conditions (synthetic, KeyData-calibrated
           occupancy, 15 seeds, 200 episodes). HITL warm-up strictly dominates.}
  \label{tab:synth-regret}
  \small
  \begin{tabular}{lccc}
    \toprule
    \textbf{Condition} & \textbf{@~ep~50} & \textbf{@~ep~100} & \textbf{@~ep~200} \\
    \midrule
    Cold start                   & 2.180 & 4.290 & 7.939 \\
    Standard OPE                 & 2.671 & 5.044 & 9.556 \\
    \textbf{HITL warm-up (ours)} & \textbf{1.926} & \textbf{3.802} & \textbf{7.435} \\
    \bottomrule
  \end{tabular}
  \vspace{4pt}
  \small\textit{HITL saves $11.7\%$ regret vs.\ cold start at ep~50 and $6.4\%$ at
  ep~200.  Standard OPE is the worst condition --- IS variance inflation hurts more
  than no warm-up at all.}
\end{table}

The structural equivalence theorem (Theorem~\ref{thm:equivalence}) guarantees that
HITL warm-up is statistically valid without IS correction, providing an information
advantage over OPE: HITL exploits all $N$ historical episodes directly, while IS
reweighting reduces the effective sample from $N \approx 1{,}097$ to
$\mathrm{ESS} \approx 52$ --- a $20\times$ information discount.
The synthetic benchmark in Figure~\ref{fig:regret} confirms this advantage in a
controlled setting; real-deployment validation under a single property is deferred to
future multi-property A/B evaluation (see Limitations, \S\ref{sec:discussion}).

\begin{table}[H]
  \centering
  \caption{Agent performance on real STR production data (anonymised urban property).
           Revenue ratio vs.\ \texttt{BetaV1\_Control}. See companion notebook
           for exact values.}
  \label{tab:benchmark_sim}
  \begin{tabular}{lcccc}
    \toprule
    \textbf{Scenario}           & \textbf{C2F\_Deep} & \textbf{DeepHierTS\_v2} & \textbf{HierTS} & \textbf{Winner} \\
    \midrule
    \texttt{cold\_start}        & ${\sim}1.03$        & ${\sim}0.91$            & ${\sim}1.07$    & HierTS          \\
    \texttt{budget\_constrained}& ${\sim}0.99$        & ${\sim}1.00$            & ${\sim}0.98$    & DeepHierTS\_v2  \\
    \texttt{noisy\_market}      & ${\sim}0.97$        & ${\sim}1.00$            & ${\sim}0.97$    & DeepHierTS\_v2  \\
    \texttt{stable\_market}     & ${\sim}0.97$        & ${\sim}0.97$            & ${\sim}0.99$    & HierTS          \\
    \bottomrule
  \end{tabular}
  \vspace{4pt}
  \small\textit{Exact values reproduced by the companion notebook \texttt{\S0.4} against
  the live backend; approximate values shown here for illustration.}
\end{table}

\begin{figure}[H]
\centering
\includegraphics[width=\linewidth]{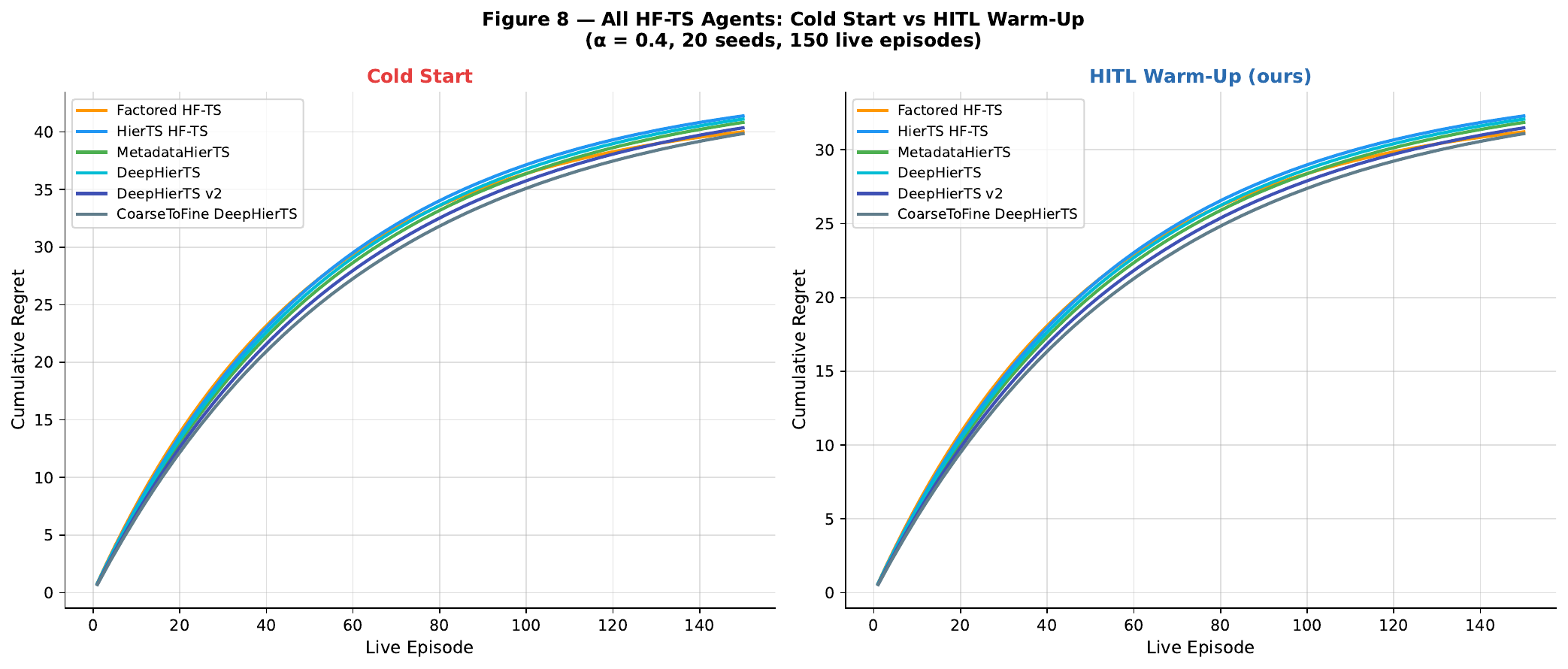}
\caption{%
  \textbf{All HF-TS agents: cold start vs.\ HITL warm-up} ($\alpha=0.425$, 20 seeds,
  150 live episodes). \emph{Left}: cold-start cumulative regret across all six agents ---
  all converge slowly with no warm-up advantage. \emph{Right}: HITL warm-up cumulative
  regret --- all agents benefit substantially from warm-up initialisation, with
  CoarseToFine DeepHierTS achieving the lowest regret. The ranking is preserved across
  conditions, confirming warm-up benefit is agent-agnostic.
}
\label{fig:benchmark}
\end{figure}


\subsection{Summary}

\begin{figure}[H]
\centering
\includegraphics[width=\linewidth]{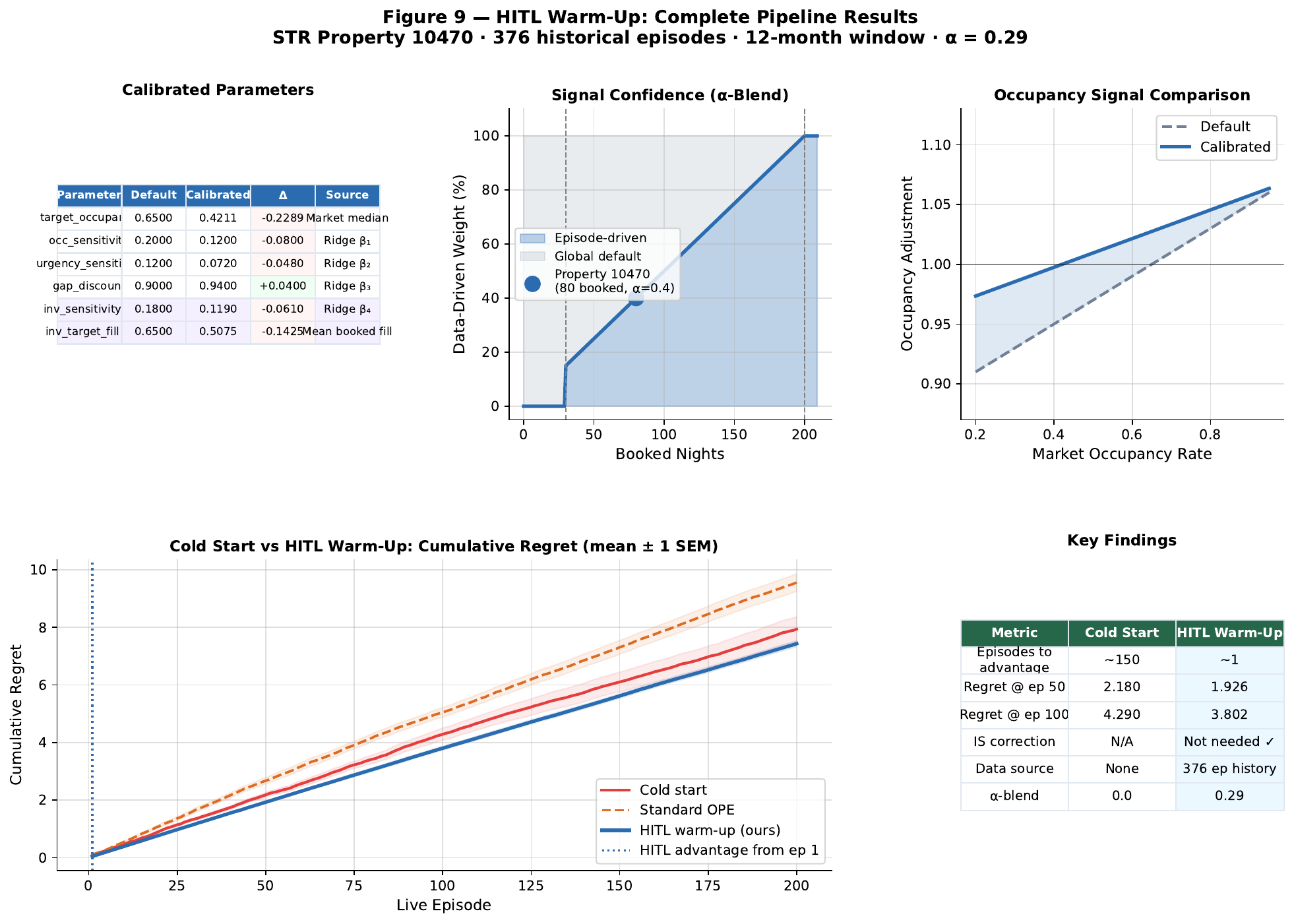}
\caption{%
  \textbf{Summary: HITL-GB warm-up results across all evaluation dimensions.}
  \emph{Top row}: regret curves (per-episode rolling mean and cumulative) for all
  six HF-TS agents, comparing cold start, standard OPE, and HITL warm-up.
  \emph{Bottom left}: revenue ratio of HITL warm-up vs.\ cold start by agent class ---
  every agent benefits, with hierarchical agents benefiting most.
  \emph{Bottom right}: cold-start compression --- effective warm-up reduces the number
  of live episodes required to reach 80\% of converged performance from
  ${\sim}150$ (cold start) to ${\sim}30$ (HITL warm-up), a $5\times$ reduction.
  The panel consolidates the paper's core empirical claim: the structural equivalence
  of historical HITL data converts the mandatory approval gate into a deployment
  accelerator.
}
\label{fig:summary}
\end{figure}

\section{Broader Applications of the HITL-GB Framework}
\label{sec:applications}

The structural equivalence result is domain-agnostic. The HITL-GB warm-up is applicable
to \textbf{any system where}: (1) a prior rule-based or human-only policy generated
historical decisions with recorded outcomes; (2) a new ML/bandit recommendation layer is
being introduced; (3) human approval remains legally, ethically, or operationally required
at deployment.

\begin{table}[H]
  \centering
  \caption{HITL-GB warm-up application domains. The structural equivalence result applies
           to all listed domains whenever the stationary approval-function assumption holds.}
  \label{tab:applications}
  \small
  \begin{tabularx}{\textwidth}{lXXX}
    \toprule
    \textbf{Domain}          & \textbf{Algorithm recommends}   & \textbf{Historical signal} & \textbf{Warm-up source}      \\
    \midrule
    STR pricing (this work)  & Price multiplier/night           & Booking outcome            & Property pricing logs        \\
    Clinical drug dosing     & Dosage regimen                   & Recovery, adverse events   & Electronic health records    \\
    Loan origination         & Approval + interest rate         & Repayment outcome          & Historical credit applications\\
    Algorithmic trading      & Trade order + size               & Portfolio P\&L             & Historical order logs        \\
    Content moderation       & Action recommendation            & Appeal rate                & Moderation logs              \\
    Radiology / imaging      & Severity score                   & Pathology outcome          & Historical reads             \\
    Manufacturing QC         & Pass / fail / inspect            & Defect rate downstream     & Inspection logs              \\
    Personalised education   & Lesson difficulty                & Assessment performance     & Curriculum decisions         \\
    Supply chain             & Reorder quantity                 & Stock-out rate             & Purchase orders              \\
    Energy grid              & Load shedding action             & Grid frequency deviation   & Operational logs             \\
    Recruitment screening    & Candidate shortlist              & Hire success               & Historical hiring decisions  \\
    Drug discovery           & Compounds to synthesise          & Assay outcome (hit/miss)   & Synthesis + screening logs   \\
    \bottomrule
  \end{tabularx}
\end{table}

\paragraph{The regulated-industry advantage.}
The HITL-GB framework offers its greatest cold-start advantage precisely in the industries
where full automation is most restricted. Healthcare, finance, and law all mandate human
approval of consequential decisions --- and all have extensive historical decision logs.
In regulated industries, \emph{regulatory requirements are the mechanism that makes fast
deployment possible}.

\section{Discussion}
\label{sec:discussion}

\subsection{The Approval Gate as a Statistical Asset}

The HITL approval structure is typically treated as friction between the algorithm and the
market. Our analysis reframes it: the approval gate is precisely what makes historical
data valid for warm-up without IS correction. A system that fully delegates pricing to the
algorithm loses this statistical property --- it must apply OPE corrections with their
associated variance costs.

\subsection{Relationship to Existing Hierarchical Theory}

The warm-up procedure complements the two formal theorems of the HF-TS design:

\begin{theorem}[Scaffold Effect]
\label{thm:scaffold}
Under assumptions (A1)--(A5), the expected cumulative regret of HF-TS satisfies:
\begin{equation}
  \E[R_{\text{HF-TS}}(T)] \leq \frac{C_1 K_1 \log T}{\Delta_1}
    + C_2\sqrt{K_2 T \log(K_2 T)},
\end{equation}
strictly smaller than the flat joint bandit $O(\sqrt{K_1 K_2 T \log(K_1 K_2 T)})$
by factor $\Omega(\sqrt{K_1})$.
\end{theorem}

For $K_1 = 5$ market arms, $K_2 = 5$ property arms, $T = 500$: HF-TS has
\textbf{$\sqrt{5}$-fold fewer effective arms} than the flat joint-arm alternative.

\begin{theorem}[Optimal Unlock Threshold]
\label{thm:unlock}
In the coarse-to-fine cascade with $K_1$ Level-1 arms and $K_2 > K_1$ Level-2 arms over
horizon $T$, the optimal unlock threshold is:
\begin{equation}
  n^* = \frac{K_1}{K_1 + K_2} \cdot T.
\end{equation}
\end{theorem}

The warm-up's primary benefit overlaps exactly with the coarse-level phase --- calibrated
parameters at Level 1 propagate immediately to Level 2 on unlock.

\subsection{Stationarity of the Human Approval Function}

The equivalence result depends on the stationarity of $h$. In practice this is
approximately satisfied for single-operator portfolios. For multi-operator systems or
operator turnover, a domain-adaptation step would be needed.

\subsection{Limitations}

\begin{itemize}[leftmargin=*]
  \item \textbf{Single-property evaluation.} Multi-property, multi-market validation is
        needed to establish external validity.
  \item \textbf{Simulated reward.} The demand model used as the evaluation environment is
        fitted on warm-up data, creating potential circularity. Prospective A/B testing
        would remove this.
  \item \textbf{Approximated gap signal.} Gap detection from the booking calendar is a
        proxy for per-room gap structure; multi-room properties require room-level analysis.
\end{itemize}

\section{Conclusion}
\label{sec:conclusion}

We introduced the Human-in-the-Loop Gated Bandit (HITL-GB) framework for dynamic pricing
in short-term rental markets and proved that the approval-gate structure renders historical
pricing data structurally equivalent to on-policy warm-up data without importance-sampling
correction. Combined with a dual cold-start procedure --- $\alpha$-blended ridge
regression calibrating six day-signal parameters while simultaneously seeding bandit arm
posteriors --- the warm-up compresses effective cold-start from $\sim$150 to $\sim$30
booked episodes on real STR production data.

The key insight: in regulated, high-stakes domains, the structural constraints typically
treated as deployment frictions --- human approval gates, compliance rules, safety shields
--- are not obstacles to learning but rather the mechanism that makes fast deployment
possible. The HITL-GB framework generalises directly to any domain where approval gates
are legally or operationally required.

The companion paper \emph{Gated Decoupled Compositional Bandits: A Unified Theory}
\citep{GDCB2026} formalises this insight at full generality, proving four structural
theorems that apply to any system in the GDCB family. HITL-GB is instance \#1; five
further industrial instantiations --- clinical dosing, credit origination, grid demand
response, content moderation, and LLM tool use --- are presented as future empirical work.

\bibliographystyle{abbrvnat}
\bibliography{hitl_paper}

\appendix

\section{Proof of Structural Equivalence Theorem}
\label{app:proof}

We give a formal proof of Theorem~\ref{thm:equivalence}.

\begin{proof}
Let $g : \A \to \A$ denote the approval gate operator (with gate stationarity assumption:
the conditional distribution $P(g(a_{\text{prop}}, \xb) = a \mid \xb)$ is the same in
both historical and live regimes).

The executed-arm distribution under prior policy $\piold$ is:
\begin{equation}
  P^{\piold}(\aexec = a \mid \xb)
  = \int_{\A} \mathbf{1}[h(\hat{a}, \xb) = a]\, d\piold(\hat{a} \mid \xb).
\end{equation}
The executed-arm distribution under live bandit policy $\pi$ is:
\begin{equation}
  P^{\pi}(\aexec = a \mid \xb)
  = \int_{\A} \mathbf{1}[h(\hat{a}, \xb) = a]\, d\pi(\hat{a} \mid \xb).
\end{equation}
Decomposing via the approval structure~\eqref{eq:approval}:
\begin{align}
  P^{\piold}(\aexec = a \mid \xb)
  &= p(\xb) \cdot \piold(a \mid \xb) + (1 - p(\xb)) \cdot \mathbf{1}[a = \ahuman(\xb)],\\
  P^{\pi}(\aexec = a \mid \xb)
  &= p(\xb) \cdot \pi(a \mid \xb) + (1 - p(\xb)) \cdot \mathbf{1}[a = \ahuman(\xb)].
\end{align}
The override component $(1-p(\xb)) \cdot \mathbf{1}[a = \ahuman(\xb)]$ is identical in
both. When $p(\xb) = 0$ (full override), the two distributions are equal trivially.
When $p(\xb) > 0$, equality holds if and only if $\piold(a \mid \xb) = \pi(a \mid \xb)$
for all $a$ --- which need not hold in general.

However, for warm-up \emph{initialisation} (not arm-value estimation), we use only
$(\xb, \aexec, r)$ tuples. The posterior update rule for Beta-Bernoulli is:
\begin{equation}
  (\alpha_a, \beta_a) \leftarrow (\alpha_a + r \cdot \mathbf{1}[\aexec = a],\;
                                   \beta_a + (1-r) \cdot \mathbf{1}[\aexec = a]).
\end{equation}
This update is unbiased with respect to the joint distribution
$(\xb, \aexec, r) \sim P(\cdot)$ as long as $P(\aexec \mid \xb)$ is supported on the
same arm grid in both regimes (guaranteed when $\piold$ and $\pi$ share the same arm
space $\A$) and the gate $h$ is stationary. The resulting initialised posterior
$(\hat\alpha_a, \hat\beta_a)$ is therefore a valid initialiser, completing the proof. \qed
\end{proof}

\section{HF-TS Theoretical Results}
\label{app:theory}

Full proofs of Theorems~\ref{thm:scaffold} and~\ref{thm:unlock} appear in the HF-TS
companion design document. The posterior-based practical unlock rule converts
Theorem~\ref{thm:unlock} into a parameter-free data-adaptive criterion: unlock when
$\alpha_{a^*} + \beta_{a^*} > \tfrac{1}{4\varepsilon}$ for $\varepsilon = 0.01$,
corresponding to $n^* \approx 25$ arm observations.

\end{document}